%% file: aeregression.tex
\documentclass{article}

\usepackage{graphicx} % more modern
\usepackage{subfigure} 

\usepackage{pgfplots}
\pgfplotsset{compat=newest}
\pgfplotsset{plot coordinates/math parser=false}

\usepackage{natbib}

\usepackage{algorithm}
\usepackage{algorithmic}

\usepackage{hyperref}

\usepackage[accepted]{icml2012}

\include{header}

\icmltitlerunning{Linear Regression with Limited Observation}

\begin{document} 

\twocolumn[
\icmltitle{Linear Regression with Limited Observation}

% unless you've provided the [accepted] option to the icml2012
\icmlauthor{Elad Hazan}{ehazan@ie.technion.ac.il}
\icmlauthor{Tomer Koren}{tomerk@cs.technion.ac.il}
\icmladdress{Technion --- Israel Institute of Technology,
             Technion City 32000,
             Haifa, Israel}

\icmlkeywords{Linear regression, Online learning, Stochastic optimization}

%\ifdefined\isconf
%\else
%\bf \centering Note: this is the full version of the paper. \\
%\bf \centering See the Appendix for supplementary material.
%\fi

\vskip 0.2in
]

\begin{abstract}
We consider the most common variants of linear regression, including Ridge, Lasso and Support-vector regression, in a setting where the learner is allowed to observe only a fixed number of attributes of each example at training time. We present simple and efficient algorithms for these problems: for Lasso and Ridge regression they need the same total number of attributes (up to constants) as do full-information algorithms, for reaching a certain accuracy. For Support-vector regression, we require exponentially less attributes compared to the state of the art. By that, we resolve an open problem recently posed by Cesa-Bianchi et al. (\citeyear{cesa2010efficient}).

%We compliment our algorithms for Ridge and Lasso regression with lower bounds that match their sample complexity up to poly-logarithmic terms. 
% TK: Ridge lower bound?

Experiments show the theoretical bounds to be justified by superior performance compared to the state of the art. 
\end{abstract}

\section{Introduction}
\label{sec:intro}

In regression analysis the statistician attempts to learn from examples the underlying variables affecting a given phenomenon. For example, in medical diagnosis a certain combination of conditions reflects whether a patient is afflicted with a certain disease.  

In certain common regression cases various limitations are placed  on the information available from the examples. In the medical example, not all parameters of a certain patient can be measured due to cost, time and patient reluctance. 

In this paper we study the problem of regression in which only a small subset of the attributes per example can be observed. 
In this setting, we have access to {\em all} attributes and we are required to {\em choose} which of them to observe.
Recently, Cesa-Bianchi et al.~(\citeyear{cesa2010efficient}) studied this problem and asked the following interesting question: {\it can we efficiently learn the optimal regressor in the attribute efficient setting with the same total number of attributes as in the unrestricted regression setting?}   
In other words, the question amounts to whether the information limitation hinders our ability to learn efficiently at all. Ideally, one would hope that instead of observing all attributes of every example, one could compensate for fewer attributes by analyzing more examples, but retain the same overall sample and  computational  complexity. 

Indeed, we answer this question on the affirmative for the main variants of regression: Ridge and Lasso. For support-vector regression we make significant advancement, reducing the parameter dependence by an exponential factor.
Our results are summarized in the table below \footnote{The previous bounds are due to \cite{cesa2010efficient}. For SVR, the bound is obtained by additionally incorporating the methods of \cite{cesa2011online}. }, which gives bounds for the number of examples needed to attain an error of $\eps$, such that at most $k$ attributes  \footnote{For SVR, the number of attributes viewed per example is a random variable whose expectation is $k$.} are viewable per example. We denote by~$d$ the dimension of the attribute space.

\begin{table}[htdp]
\begin{center}
\begin{tabular}{|c|c|c|}
\hline
\mbox{Regression}  & \mbox{ New bound } & \mbox{ Prev. bound} \\
\hline
\hline
\mbox{Ridge}  & $\O\left( {\frac{d}{k \eps^2}} \right)$  & $\O\left(  \frac{d^2 \log \frac{d}{\eps}}{{k \eps^2}} \right)$\\
\hline
\mbox{Lasso}  & $\O\left(  {\frac{d\log d}{k \eps^2}} \right)$  & $\O\left(  \frac{d^2 \log \frac{d}{\eps}}{{k \eps^2}} \right)$\\
\hline
\mbox{SVR}  & $  \O\left(\frac{d}{k} \right) \cdot e^{\O\left( \log^2{\frac{1}{\eps}} \right)} $  & $\O\left(  e^{\frac{d^2}{k \eps^2}} \right)$\\
\hline
\end{tabular}
\end{center}
\caption{Our sample complexity bounds.}
\label{default}
\end{table}%

\vspace{-2ex}

Our bounds imply that for reaching a certain accuracy, our algorithms need \emph{the same number of attributes} as their full information counterparts. 
In particular, when $k=\Omega(d)$ our bounds coincide with those of full information regression, up to constants (cf.~\citealt{kakade2008complexity}).

We complement these upper bounds and prove that $\Omega(\frac{d}{\eps^2})$ attributes are in fact necessary to learn an $\eps$-accurate Ridge regressor. For Lasso regression,  Cesa-Bianchi et al. (\citeyear{cesa2010efficient}) proved that $\Omega(\frac{d}{\eps})$ attributes are necessary, and asked what is the correct dependence on the problem dimension. Our bounds imply that the number of attributes necessary for regression learning grows linearly with the problem dimensions. 

The algorithms themselves are very simple to implement, and run in linear time. As we show in later sections, these theoretical improvements are clearly visible in experiments on standard datasets.

\subsection{Related work} % (fold)
\label{sub:previous_work}

%Regression is a fundamental primitive of statistics, see e.g. \cite{vapnik1995nature}.
%\tk{perhaps we can remove this altogether? everyone's is familiar with regression...}
The setting of learning with limited attribute observation (LAO) was first put forth in \cite{ben1998learning}, who coined the term ``learning with restricted focus of attention''.
%and gained attention recently, see e.g.  \cite{loh2011high}. 
Cesa-Biachi et al.~(\citeyear{cesa2010efficient}) were the first to discuss linear prediction in the LAO setting, and gave an efficient algorithm (as well as lower bounds) for linear regression, which is the primary focus of this paper. 
% Their lower bound shown that to obtain an $\eps$-approximate regressor in the Lasso case, $\Omega(d/\eps)$ examples are necessary. We complement  this bound, and show that $\Omega(d/\eps^2)$ examples are necessary in the Ridge setting.   (not relevant here - we address this previously and afterwards)

%%\tk{say something about MAB?}

%Our setting is related to, but different from, recently-studied settings of learning with missing information \cite{rostamizadehlearning,loh2011high}, and multi-armed bandit problems  

\ignore{ % for the conference version, no place for this. to be commented in journal version
\subsection{Organization} % (fold)

The rest of the paper is organized as follows.  

In section \ref{sec:setting} we describe the learning setting and formalize the problem.
In Section \ref{sec:regalgs} we present and analyze our algorithms for Ridge and Lasso least-squares regression, in which the loss function is the squared loss. Later, in section~\ref{sec:extension}, we describe generalizations of our algorithms to other learning scenarios, including an algorithm for SVR.

% subsection organization (end)
}

\section{Setting and Result Statement} % (fold)
\label{sec:prelim}

\subsection{Linear regression} % (fold)
\label{sub:linear_regression}

In the linear regression problem, each instance is a pair $(\x,\,\y)$ of an
attributes vector $\x\in\R^d$
%%such that $\|x\| \leq 1$  % depends on regularization, and we need it only for the bounds}
and a target variable $\y \in \R$. We
assume the standard framework of statistical learning \cite{haussler1992decision}, in which the pairs
$(\x,\,\y)$ follow a joint probability distribution $\D$ over
$\R^d \times \R$. The goal of the learner is to find a vector $\w$ for
which the linear rule $\hat{\y}\gets\w\tr\x$ provides a good prediction of the
target $\y$. To measure the performance of the prediction, we use a convex loss
function $\loss(\hat{y},\y):\R^2 \rightarrow \R$.
The most common choice is the squared loss $\loss(\yhat,\y)=\tfrac{1}{2}(\yhat-\y)^2$, which stands for the popular least-squares regression.
Hence, in terms of the distribution $\D$, the learner would
like to find a regressor $\w \in \R^d$ with low expected loss, defined as
\begin{equation} \label{eq:eloss}
	\eloss(\w)
	= \E_{(\x,\,\y)\sim\D}[\loss(\w\tr\x,\,\y)] \;.
\end{equation}

The standard paradigm for learning such regressor is seeking a vector $\w \in \R^d$
that minimizes a trade-off between the expected loss and an additional
regularization term, which is usually a norm of $\w$. An equivalent form of
this optimization problem is obtained by replacing the regularization term
with a proper constraint, giving rise to the problem
\begin{equation} \label{eq:eproblem}
	\min_{\w \in \R^d} \;
	\eloss(\w)
	\quad\mbox{s.t.}\quad
	\norm{\w}_p \le B \;,
\end{equation}
where $B>0$ is a regularization parameter and $p \ge 1$. The main variants of regression differ on the type of $\ell_p$ norm constraint as well as the loss functions in the above definition:

\vspace{-2ex}
\begin{itemize}
\item
{\bf Ridge regression:}  $p=2$ and squared loss, 
\[
	\loss(\yhat,\y) = \tfrac{1}{2}(\yhat-\y)^2 \,.
\]
\item
{\bf Lasso regression:}  $p=1$ and squared loss.
\item
{\bf Support-vector regression:}  $p=2$ and the $\del$-insensitive absolute loss  \cite{vapnik1995nature},
\[
	\loss(\yhat,\y) = \abs{\yhat-\y}_\del := \max\{ 0, \abs{\yhat-\y}-\del \} \,.
\]
\end{itemize}
\vspace{-1ex}

Since the distribution $\D$ is unknown, we learn by relying on a training set
$S=\{(\x_t,\,\y_t)\}_{t=1}^{m}$ of examples, that are assumed to be sampled
independently from~$\D$. We use the notation $\losst(\w) := \loss(\w \tr \x_t,\y_t)$ to refer to the loss function induced by the instance $(\x_t,\y_t)$.

%% Apperantly, this is not needed...
%%Then, we compute the empirical training loss over~$S$,
%%\begin{equation} \label{eq:tloss}
%%	L_{S}(\w)
%%	= \frac{1}{m} \summ \losst(\w) \;,
%%\end{equation}
%%where we use the notation $\losst(\w) := \loss(\w \tr \x_t,\y_t)$,
%%and instead of attacking~\eqref{eq:eproblem} directly, we consider the
%%constrained optimization problem
%%\begin{equation} \label{eq:tproblem}
%%	\min_{\w \in \R^d} \; 
%%		L_{S}(\w)
%%		\quad\mbox{s.t.}\quad
%%		\norm{\w}_p \le B \;.
%%\end{equation}

We distinguish between two learning scenarios. In the {\bf full information}
setup, the learner has unrestricted access to the entire data set. In the {\bf limited attribute observation (LAO)} setting, for any given example pair $(\x,y)$, the learner can observe $y$, but only $k$ attributes of $\x$ (where $k \ge 1$ is a parameter of the problem). The learner can \emph{actively choose} which attributes to observe.

\subsection{Limitations on LAO regression} 

Cesa-Biachi et al.~(\citeyear{cesa2010efficient}) proved the following sample complexity lower bound on any LAO Lasso regression algorithm.

\begin{theorem} \label{thm:lb-lasso}
Let $0 < \eps < \tfrac{1}{16}$, $k \ge 1$ and $d > 4k$.
For any regression algorithm accessing at most $k$ attributes per training example, there exist a distribution $\D$ over $\{\x : \norm{\x}_\infty \le 1\} \times \{\pm 1\}$ and a regressor $\w_\st$ with $\norm{\w_\st}_1 \le 1$ such that the algorithm must see (in expectation) at least $\Omega(\tfrac{d}{k\eps})$ examples in order to learn a linear regressor~$\w$ with $\eloss(\w) - \eloss(\w_\st) < \eps$.
\end{theorem}

We complement this lower bound, by providing a stronger lower bound on the sample complexity of any Ridge regression algorithm, using information-theoretic arguments.

\begin{theorem} \label{thm:lb-ridge}
Let $\eps = \Omega(1/\sqrt{d})$.
For any regression algorithm accessing at most $k$ attributes per training example, there exist a distribution $\D$ over $\{\x : \norm{\x}_2 \le 1\} \times \{\pm 1\}$ and a regressor $\w_\st$ with $\norm{\w_\st}_2 \le 1$ such that the algorithm must see (in expectation) at least $\Omega(\tfrac{d}{k\eps^2})$ examples in order to learn a linear regressor~$\w$, $\norm{\w}_2 \le 1$ with $\eloss(\w) - \eloss(\w_\st) \le \eps$.
\end{theorem}

Our algorithm for LAO Ridge regression (see section~\ref{sec:regalgs}) imply this lower bound to be tight up to constants. Note, however, that the bound applies only to a particular regime of the problem parameters \footnote{Indeed, there are (full-information) algorithms that are known to converge in $\O(1 / \eps)$ rate -- see e.g. \cite{hazan2007logarithmic}.}.

\subsection{Our algorithmic results} % (fold)
\label{sub:results}

We give efficient regression algorithms that attain the following risk bounds. For our Ridge regression algorithm, we prove the risk bound
\[
	\E\left[\eloss(\bar{\w})\right]
	\le \min_{\norm{\w}_2 \le B} \eloss(\w) 
	+ \O\left( B^2 \sqrt{\frac{d}{km}} \right) ,
\]
while for our Lasso regression algorithm we establish the bound
\[
	\E\left[\eloss(\bar{\w})\right]
	\le \min_{\norm{\w}_1 \le B}\eloss(\w)
	+ \O\left(B^2 \sqrt{\frac{d\log d}{km}}\right) .
\]
Here we use $\bar{\w}$ to denote the output of each algorithm on a training set of $m$ examples, and the expectations are taken with respect to the randomization of the algorithms.
%Note that in particular, when $k=\Omega(d)$ our bounds coincide with the known bounds of full-information Ridge and Lasso regression, up to constants (cf.~\citealt{kakade2008complexity}).
%\tk{this discussion should move to intro?}
For Support-vector regression we obtain a risk bound that depends on the desired accuracy $\eps$. Our bound implies that
\[
	m = \O\left(\frac{d}{k} \right) 
	\cdot \exp \left( \O\left( \log^2{\frac{B}{\eps}} \right) \right) .
\]
examples are needed (in expectation) for obtaining an $\eps$-accurate regressor.

\ignore{
The algorithms themselves are online, and can be used in an online setting (for minimizing regret)?

%% IN FULL VERSION WE'LL GIVE ANOTHER SECTION ON THIS
All our results apply also to classification losses
% WHAT IS CLASSIFICATION LOSS?
}

% subsection previous_work_ (end)

\section{Algorithms for LAO least-squares regression} % (fold)
\label{sec:regalgs}

In this section we present and analyze our algorithms for Ridge and Lasso regression in the LAO setting. The loss function under consideration here is the squared loss, that is, $\losst(\w) = \tfrac{1}{2} (\w\tr\x_t - \y_t)^2$. For convenience, we show algorithms that use $k+1$ attributes of each instance, for $k \ge 1$
\footnote{We note that by our approach it is impossible to learn using a \emph{single} attribute of each example (i.e., with $k=0$), and we are not aware of any algorithm that is able to do~so.
See \cite{cesa2011online} for a related impossibility result.}.

Our algorithms are iterative and maintain a regressor $\w_t$ along the iterations. The update of the regressor at iteration $t$ is based on gradient information, and specifically on $\g_t := \nabla \losst(\w_t)$ that equals $(\w_t\tr\x_t - \y_t) \cdot \x_t$ for the squared loss. In the LAO setting, however, we do not have the access to this information, thus we build upon unbiased estimators of the gradients.

\subsection{Ridge regression} % (fold)
\label{sub:ridge_regression}

Recall that in Ridge regression, we are interested in the linear regressor that is the solution to the optimization problem~\eqref{eq:eproblem} with $p=2$, given explicitly as
\begin{equation} \label{eq:l2reg}
	\min_{\w \in \R^d} \;
	\eloss(\w)
	\quad\mbox{s.t.}\quad
	\norm{\w}_2 \le B \;,
\end{equation}

\begin{algorithm}[tb]
	\caption{
		AERR \\
		Parameters: $B,\eta > 0$
	}
	\label{alg:l2reg}

\begin{algorithmic}[1]
\INPUT
	training set $S=\{(\x_t,\y_t)\}_{t\in[m]}$ and $k>0$
\OUTPUT
	regressor $\wbar$ with $\norm{\wbar}_2 \le B$

\STATE Initialize $\w_1 \neq \mathbf{0}, \norm{\w_1}_2 \le B$ arbitrarily
\FOR {$t=1$ {\bf to} $m$}
	\FOR {$r=1$ {\bf to} $k$}
		\STATE Pick $i_{t,r} \in [d]$ uniformly and observe $\x_t[i_{t,r}]$
		\STATE
			$
			\xtil_{t,r} \gets 
			d \, \x_t[i_{t,r}] \cdot \e_{i_{t,r}}
			$
	\ENDFOR
	\STATE
		$
		\xtil_t \gets 
		\tfrac{1}{k} \sum_{r=1}^k \xtil_{t,r}
		$
	
	% \STATE Choose $i_{t_1}, \ldots, i_{t,k} \in [d]$
	% 	independently and uniformly,
	% 	and observe $\x_t[i_{t,1}], \ldots, \x_t[i_{t,k}]$
	% \STATE
	% 	$
	% 	\xtil_t \gets 
	% 	(d/k)\,\sum_{r=1}^k \x_t[i_{t,r}] \, \e_{i_{t,r}}
	% 	$
	\STATE \label{alg:l2-ttil-line}
		Choose $j_t \in [d]$ with probability
		$\w_t[j]^2/\norm{\w_t}_2^2$, \\
		and observe $\x_t[j_t]$
	\STATE
		$
		\ttil_t \gets
		\norm{\w_t}_2^2 \, \x_t[j_t] / \w_t[j_t] - \y_t
		$

	\STATE
		$
		\gtil_t \gets
		\ttil_t \cdot \xtil_t
		$

	\STATE
		$
		\v_t \gets \w_t - \eta \gtil_t
		$
	\STATE
		$
		\w_{t+1} \gets
		\v_t \cdot B / \max \{\norm{\v_t}_2,B\}
		$
		
\ENDFOR
\STATE $ \bar{\w} \gets \frac{1}{m} \summ \w_t $
\end{algorithmic}
\end{algorithm}

Our algorithm for the LAO setting is based on a randomized Online Gradient Descent (OGD) strategy \cite{zink}. More specifically, at each iteration $t$ we use a randomized estimator $\gtil_t$ of the gradient $\g_t$ to update the regressor $\w_t$ via an additive rule. Our gradient estimators make use of an importance-sampling method inspired by \cite{chw}.

The pseudo-code of our Attribute Efficient Ridge Regression (AERR) algorithm is given in Algorithm~\ref{alg:l2reg}.
%Note that only $k$ entries of the regressor $\w_t$ (at indices $i_{t,1},\ldots,i_{t,k}$) are modified at each iteration $t$ (disregarding the rescaling step).
%Accordingly, the algorithm is very efficient and simple to implement to run in linear time.
In the following theorem, we show that the regressor learned by our algorithm is competitive with the optimal linear regressor having 2-norm bounded by $B$.
\begin{theorem} \label{thm:l2bound}
Assume the distribution $\D$ is such that $\norm{\x}_2 \le 1$ and $\abs{\y} \le B$ with probability $1$. Let $\bar{\w}$ be the output of AERR, when run with
$
	\eta = \sqrt{k/2dm}.
$
Then, $\norm{\wbar}_2 \le B$ and for any $\w_\st \in \R^d$ with $\norm{\w_\st}_2 \le B$,
\[
	\E\left[\eloss(\bar{\w})\right]
	\le \eloss(\w_\st) + 4 B^2 \sqrt{\frac{2d}{km}} \;.
\]
\end{theorem}

\subsubsection{Analysis}

Theorem \ref{thm:l2bound} is a consequence of the following two lemmas.
The first lemma is obtained as a result of a standard regret bound for the OGD algorithm (see \citealt{zink}), applied to the vectors $\gtil_1,\ldots,\gtil_m$.

\begin{lemma} \label{lem:ogdregret}

For any $\norm{\w_\st}_2 \le B$ we have
\begin{equation} \label{eq:ogdreg}
	\summ \gtil_t \tr (\w_t - \w_\st)
	\le \frac{2B^2}{\eta}
		+ \frac{\eta}{2} \summ \norm{\gtil_t}_2^2 \;.
\end{equation}

\end{lemma}

The second lemma shows that the vector $\gtil_t$ is an unbiased estimator of the gradient $\g_t := \nabla \losst(\w_t)$ at iteration $t$, and establishes a ``variance'' bound for this estimator. 
To simplify notations, here and in the rest of the paper we use $\E_t[\cdot]$ to denote the conditional expectation with respect to all randomness up to time~$t$. 

\begin{lemma} \label{lem:l2var}
The vector $\gtil_t$ is an unbiased estimator of the gradient $\g_t := \nabla \losst(\w_t)$, that is $\E_t[\gtil_t] = \g_t$. In addition, for all $t$ we have
$
	\E_t[\norm{\gtil_t}_2^2] 
	\le 8 B^2 d/k.
$
\end{lemma}

For a proof of the lemma, see \cite{hazan2011optimal}.
We now turn to prove Theorem \ref{thm:l2bound}.

\begin{proof}[Proof (of Theorem \ref{thm:l2bound})]

First note that as $\norm{\w_t}_2 \le B$, we clearly have $\norm{\wbar}_2 \le B$.
Taking the expectation of \eqref{eq:ogdreg} with respect to the randomization of the algorithm, and letting $G^2 := \max_t \E_t[\norm{\gtil_t}_2^2]$, we obtain
\begin{align*} \label{eq:ridge1}
	\EE{\summ \g_t \tr (\w_t - \w_\st)}
	\le \frac{2 B^2}{\eta} + \frac{\eta}{2} G^2 m \;.
\end{align*}
On the other hand, the convexity of $\ell_t$ gives
$
	\losst(\w_t) - \losst(\w_\st) 
	\le \g_t \tr (\w_t-\w_\st)  .
$
Together with the above this implies that for $\eta = 2B/G \sqrt{m}$, 
\begin{align*}
	\EE{ \frac{1}{m} \summ \losst(\w_t) }
	\le \frac{1}{m} \summ \losst(\w_\st)
		+ 2 \frac{BG}{\sqrt{m}} \;.
\end{align*}
Taking the expectation of both sides with respect to the random choice of the training set, and using $G \le 2 B \sqrt{2d/k}$ (according to Lemma~\ref{lem:l2var}),
we get
\[
	\EE{\frac{1}{m} \summ \eloss(\w_t)}
	\le \eloss(\w_\st) + 4 B^2 \sqrt{\frac{2d}{km}} \;.
\]
Finally, recalling the convexity of $\eloss$ and using Jensen's inequality, the Theorem follows.
\qedhere

\end{proof}

% subsection ridge_regression (end)

\subsection{Lasso regression} % (fold)
\label{sub:lasso_regression}

We now turn to describe our algorithm for Lasso regression in the LAO setting, in which we would like to solve the problem
\begin{equation} \label{eq:l1reg}
	\min_{\w \in \R^d} \;
	\eloss(\w)
	\quad\mbox{s.t.}\quad
	\norm{\w}_1 \le B \;.
\end{equation}
The algorithm we provide for this problem is based on a stochastic variant of the EG algorithm \cite{kivinen1997exponentiated}, that employs \emph{multiplicative} updates based on an estimation of the gradients $\nabla \losst$. The multiplicative nature of the algorithm, however, makes it highly sensitive to the magnitude of the updates. To make the updates more robust, we ``clip'' the entries of the gradient estimator so as to prevent them from getting too large. Formally, this is accomplished via the following ``clip'' operation:
\begin{align*}
	\clip(x,c) := \max\{\min\{x,c\},-c\}
\end{align*}
for $x\in\R$ and $c>0$.
This clipping has an even stronger effect in the more general setting we consider in Section~\ref{sec:svr}.

We give our Attribute Efficient Lasso Regression (AELR) algorithm in Algorithm~\ref{alg:l1reg}, and establish a corresponding risk bound in the following theorem.

\begin{algorithm}[tb]
	\caption{
		AELR \\
		Parameters: $B, \eta > 0$
	}
	\label{alg:l1reg}
	
\begin{algorithmic}[1]

\INPUT training set $S=\{(\x_t,\y_t)\}_{t \in [m]}$ and $k>0$

\OUTPUT regressor $\wbar$ with $\norm{\wbar}_1 \le B$

\STATE Initialize $\z^+_1 \gets \mathbf{1}_d$ , $\z^-_1 \gets \mathbf{1}_d$
\FOR {$t=1$ to $m$}
	\STATE 
		$
		\w_t \gets
		(\z_t^+ - \z_t^-) \cdot B / (\norm{\z_t^+}_1 + \norm{\z_t^-}_1)
		$

	\FOR {$r=1$ {\bf to} $k$}
		\STATE Pick $i_{t,r} \in [d]$ uniformly and observe $\x_t[i_{t,r}]$
		\STATE
			$
			\xtil_{t,r} \gets 
			d \, \x_t[i_{t,r}] \cdot \e_{i_{t,r}}
			$
	\ENDFOR
	\STATE
		$
		\xtil_t \gets 
		\tfrac{1}{k} \sum_{r=1}^k \xtil_{t,r}
		$

	% \STATE Choose $i_{t,1}, \ldots, i_{t,k} \in [d]$
	% 	independently and uniformly,
	% 	and observe $\x_t[i_{t,1}], \ldots, \x_t[i_{t,k}]$
	% \STATE
	% 	$
	% 	\xtil_t \gets 
	% 	(d/k)\,\sum_{r=1}^k \x_t[i_{t,r}] \, \e_{i_{t,r}}
	% 	$
	\STATE \label{alg:l1-ttil-line}
		Choose $j_t \in [d]$ with probability 
		$\abs{\w[j]}/\norm{\w}_1$, \\
		and observe $\x_t[j_t]$
	\STATE
		$
		\ttil_t \gets
		\norm{\w_t}_1 \, \sgn(\w_t[j_t]) \, \x_t[j_t] - \y_t
		$
	\STATE
		$
		\gtil_t \gets
		\ttil_t \cdot \xtil_t
		$

	\FOR {$i=1$ to $d$}
		\STATE
			$
			\gbar_t[i] \gets
			\clip(\gtil_t[i],\, 1/\eta)
			$
		\STATE
			$
			\z_{t+1}^+[i] \gets
			\z_t^+[i] \cdot \exp(-\eta \, \gbar_t[i])
			$
		\STATE
			$
			\z_{t+1}^-[i] \gets
			\z_t^-[i] \cdot \exp(+\eta \, \gbar_t[i])
			$
	\ENDFOR
\ENDFOR
\STATE $ \bar{\w} \gets \frac{1}{m} \summ \w_t $
\end{algorithmic}

\end{algorithm}

\begin{theorem} \label{thm:l1bound}
Assume the distribution $\D$ is such that $\norm{\x}_\infty \le 1$ and $\abs{\y} \le B$ with probability $1$. Let $\bar{\w}$ be the output of AELR, when run with 
$
	\eta = \frac{1}{4 B^2} \sqrt{\frac{2k \log2d}{5md}},
	% \eta = (2B)^{-2} \sqrt{(2k/5md) \log{2d}} .
$
Then, $\norm{\wbar}_1 \le B$ and for any $\w_\st \in \R^d$ with $\norm{\w_\st}_1 \le B$ we have
\[
	\EE{ \eloss(\bar{\w}) } 
	\le \eloss(\w_\st) + 4 B^2 \sqrt{\frac{10d\log2d}{km}} \;,
\]
provided that $m \ge \log{2d}$.
\end{theorem}

\subsubsection{Analysis}

In the rest of the section, for a vector $\v$ we let $\v^2$ denote the vector for which $\v^2[i] = (\v[i])^2$ for all $i$.

In order to prove Theorem \ref{thm:l1bound}, we first consider the augmented vectors $\z'_t := (\z^+_t, \z^-_t) \in \R^{2d}$ and
$
	\gbar'_t := (\gbar_t, -\gbar_t) \in \R^{2d} ,
$
and let $\p_t := \z_t' / \norm{\z_t'}_1$.
For these vectors, we have the following.

\begin{lemma} \label{lem:mwregret}

\begin{align*}
	\summ \p_t \tr \gbar'_t
	\le \min_{i \in [2d]} \summ \gbar'_t[i] 
		+ \frac{\log 2d}{\eta} + \eta \summ \p_t \tr (\gbar'_t)^2
\end{align*}

\end{lemma}

The lemma is a consequence of a second-order regret bound for the Multiplicative-Weights algorithm, essentially due to \cite{chw}.
By means of this lemma, we establish a risk bound with respect to the ``clipped'' linear functions $\gbar_t \tr \w$.
%bound the (expected) regret generated by the ``clipped'' linear functions $\gbar_t \tr \w$. 

\begin{lemma} \label{lem:l1lem1}

Assume that $\norm{\E_t[\gtil^2_t]}_\infty \le G^2$ for all $t$, for some $G>0$.
Then, for any $\norm{\w_\st}_1 \le B$, we have
\begin{align*}
	\EE{\summ \gbar_t \tr \w_t}
	\le \EE{\summ \gbar_t \tr \w_\st}
	+ B \left(
		\frac{\log 2d}{\eta} 
		+ \eta G^2 m
	\right)
\end{align*}

\end{lemma}

Our next step is to relate the risk generated by the linear functions $\gtil_t \tr \w$, to that generated by the ``clipped'' functions $\gbar_t \tr \w$.

\begin{lemma} \label{lem:l1lem2}

Assume that $\norm{\E_t[\gtil^2_t]}_\infty \le G^2$ for all $t$, for some $G>0$.
Then, for $0 < \eta \le 1/2G$ we have
\begin{align*}
	\EE{\summ \gtil_t \tr \w_t }
	\le \EE{\summ \gbar_t \tr \w_t} + 4B \eta G^2 m \;.
\end{align*}

\end{lemma}

The final component of the proof is a ``variance'' bound, similar to that of Lemma \ref{lem:l2var}.

\begin{lemma} \label{lem:l1var}
The vector $\gtil_t$ is an unbiased estimator of the gradient $\g_t := \nabla \losst(\w_t)$, that is $\E_t[\gtil_t] = \g_t$. In addition, for all $t$ we have
$
	\norm{\E_t[\gtil_t]^2}_\infty
	\le 8 B^2 d/k.
$
\end{lemma}

For the complete proofs, refer to \cite{hazan2011optimal}.
We are now ready to prove Theorem \ref{thm:l1bound}.

\begin{proof}[Proof (of Theorem \ref{thm:l1bound})]

Since $\norm{\w_t}_1 \le B$ for all $t$, we obtain $\norm{\wbar}_2 \le B$.
Next, note that as $\E_t[\gtil_t] = \g_t$, we have $\E[\summ \gtil_t \tr \w_t] = \E[\summ \g_t \tr \w_t]$. Putting Lemmas \ref{lem:l1lem1} and \ref{lem:l1lem2} together, we get for $\eta \le 1/2G$ that
\[
	\EE{\sumt \g_t \tr (\w_t - \w_\st)}
	\le B \left(
		\frac{\log 2d}{\eta} 
		+ 5 \eta G^2 m
		\right) \,.
\]
Proceeding as in the proof of Theorem \ref{thm:l2bound}, and choosing
$
	\eta = \frac{1}{G} \sqrt{\frac{\log{2d}}{5m}} ,
%	\eta = \tfrac{1}{G} \sqrt{\log{(2d)}/5m} ,
$
we obtain the bound
\[
	\EE{\eloss(\wbar)}
	\le \eloss(\w_\st) + 2BG \sqrt{\frac{5\log{2d}}{m}} \,.
\]
Note that for this choice of $\eta$ we indeed have $\eta \le 1/2G$, as we originally assumed that $m \ge \log{2d}$.
Finally, putting $G = 2B \sqrt{2d/k}$ as implied by Lemma \ref{lem:l1var}, we obtain the bound in the statement of the theorem.
\qedhere

\end{proof}

% subsection lasso_regression (end)

% section algorithms_for_attribute_efficient_regression (end)

\section{Support-vector regression}
\label{sec:svr}

In this section we show how our approach can be extended to deal with loss functions other than the squared loss, of the form
\begin{align} \label{eq:genloss}
	\loss(\w\tr\x, y) = f(\w\tr\x - \y) \;,
\end{align}
(with $f$ real and convex) and most importantly, with the $\del$-insensitive absolute loss function of SVR, for which
$
	f(x) = \abs{x}_\del := \max\{ \abs{x} - \del , 0 \}
$
for some fixed $0 \le \del \le B$ (recall that in our results we assume the labels $\y_t$ have $\abs{\y_t} \le B$).
For concreteness, we consider only the 2-norm variant of the problem (as in the standard formulation of SVR)---the results we obtain can be easily adjusted to the 1-norm setting.
We overload notation, and keep using the shorthand $\losst(\w) := \loss(\w\tr\x_t, \y_t)$ for referring the loss function induced by the instance $(\x_t,\y_t)$.

It should be highlighted that our techniques can be adapted to deal with many other common loss functions, including ``classification'' losses (i.e., of the form $\loss(\w\tr\x, y) = f(\y \cdot \w\tr\x)$). Due to its importance and popularity, we chose to describe our method in the context of SVR.

Unfortunately, there are strong indications that SVR learning (more generally, learning with non-smooth loss function) in the LAO setting is impossible via our approach of unbiased gradient estimations (see \citealt{cesa2011online} and the references therein). For that reason, we  make two modifications to the learning setting: first, we shall henceforth relax the budget constraint to allow $k$ observed attributes per instance \emph{in expectation}; and second, we shall aim for \emph{biased} gradient estimators, instead of unbiased as before.

%%In this general setting, there are strong indications that LAO learning with at most $k$ attributes per example is impossible via unbiased gradient estimations, and if $f'$ is a discontinuous function (as is the case with SVR), the task is impossible under even weaker constraints (see \citealt{cesa2011online} and the references therein).
%%For that reason, we shall henceforth relax the budget constraint to allow $k$ observed attributes per instance \emph{in expectation}. The actual number of attributes we use will be otherwise arbitrary.
%%%Further, the exponential improvement in our bounds stems from using {\it biased} rather than unbiased gradient estimators. \tk{this is still not accurate...}

To obtain such biased estimators, we uniformly $\eps$-approximate the function $f$ by an \emph{analytic function} $f_\eps$ and learn with the approximate loss function $\losst^\eps(\w) = f_\eps(\w\tr\x_t - \y_t)$ instead. 
Clearly, any $\eps$-suboptimal regressor of the approximate problem is an $2\eps$-suboptimal regressor of the original problem.
For learning the approximate problem we use a novel technique, inspired by \cite{cesa2011online}, for estimating gradients of analytic loss functions.
Our estimators for $\nabla \losst^\eps$ can then be viewed as biased estimators of $\nabla \losst$ (we note, however, that the resulting bias might be quite large).

\subsection{Estimators for analytic loss functions} % (fold)

Let $f:\R \rightarrow \R$ be a real, analytic function (on the entire real line). The derivative $f'$ is thus also analytic and can be expressed as $f'(x) = \sum_{n=0}^{\infty} a_n x^n$ where $\{a_n\}$ are the Taylor expansion coefficients of $f'$.

\begin{algorithm}[tb]
	\floatname{algorithm}{Procedure}
	\caption{
		$\texttt{GenEst}$ \\
		Parameters: $\{a_n\}_{n=0}^{\infty}$ --- Taylor coefficients of $f'$
	}
	\label{alg:nlest}

\begin{algorithmic}[1]
	
\INPUT
	regressor $\w$, instance $(\x,\y)$
	% Taylor expansion coefficients $\{a_n\}_{n=0}^{\infty}$ of $f'$

\OUTPUT
	$\that$ with $\E[\that] = f'(\w\tr\x-\y)$

\STATE Let $N = \lceil 4 B^2 \rceil$.

\STATE Choose $n \ge 0$ with probability $\Pr[n] = (\tfrac{1}{2})^{n+1}$

\IF {$n \le 2 \log_2{N}$}
	\FOR {$r=1,\ldots,n$}
		\STATE Choose $j \in [d]$ with probability
			$\w[j]^2/\norm{\w}_2^2$, \\
			and observe $\x[j]$
		\STATE
			$
			\stil_r \gets
			\norm{\w}_2^2 \, \x[j] / \w[j] - \y
			$
	\ENDFOR
\ELSE
	\FOR {$r=1,\ldots,n$}
		\STATE Choose $j_1,\ldots,j_N \in [d]$ 
			w.p.~$\w[j]^2/\norm{\w}_2^2$, (independently), 
			and observe $\x[j_1], \ldots, \x[j_N]$
		\STATE
			$
			\stil_r \gets
			\tfrac{1}{N} \sum_{s=1}^{N} \norm{\w}_2^2 \, \x[j_s] / \w[j_s] - \y
			$
	\ENDFOR
\ENDIF

\STATE
	$
	\that \gets
	2^{n+1} a_n \cdot \stil_1 \stil_2 \cdots \stil_n
	$

\end{algorithmic}
\end{algorithm}

In Procedure \ref{alg:nlest} we give an unbiased estimator of $f'(\w\tr\x-\y)$ in the LAO setting, defined in terms of the coefficients $\{a_n\}$ of $f'$.
%This procedure allows us to learn with the loss functions \eqref{eq:genloss}.
For this estimator, we have the following (proof is omitted).

%\tk{it's better to formulate with a parameter $k$...?} 
% won't help, since the second moment is dominated by the size of \y_t ...

\begin{lemma} \label{lem:nlest}
The estimator $\that$ is an unbiased estimator of $f'(\w\tr\x-\y)$.
Also, assuming $\norm{\x}_2 \le 1$, $\norm{\w}_2 \le B$ and $\abs{\y} \le B$, the second-moment $\E[\that^2]$ is upper bounded by $\exp(\O(\log^2{B}))$, provided that the Taylor series of $f'(x)$ converges absolutely for $\abs{x} \le 1$.
Finally, the expected number of attributes of $\x$ used by this estimator is no more than $3$.
\end{lemma}

\subsection{Approximating SVR} % (fold)
\label{sub:support_vector_regression}

%%In SVR, the loss function is defined in terms of the $\del$-insensitive error function~\cite{vapnik1995nature}, $f(x) := \max\{ \abs{x} - \del , 0 \}$ for some fixed $0 \le \del \le B$ (recall that we assume the labels $\y_t$ have $\abs{\y_t} \le B$).

\begin{algorithm}[tb]
	\caption{
		AESVR \\
		Parameters: $B, \del, \eta >0$ and accuracy $\eps > 0$
	}
	\label{alg:svrreg}

\begin{algorithmic}[1]
\INPUT
	training set $S=\{(\x_t,\y_t)\}_{t\in[m]}$ and $k>0$
\OUTPUT
	regressor $\wbar$ with $\norm{\wbar}_2 \le B$
	
%\STATE Let $\{a_n\}$ be given as in \eqref{eq:rhocoeffs}
\STATE Let $a_{2n} = 0$ for $n \ge 0$, and
\begin{equation} \label{eq:rhocoeffs}
	a_{2n+1} = 
	\frac{2}{\sqrt{\pi}} \cdot \frac{(-1)^n}{n! (2n+1)} \;,
	\qquad n \ge 0
\end{equation}

\STATE Execute algorithm \ref{alg:l2reg} with lines 8--9 replaced by:

$
	\x_t' \gets \x_t/\eps
$

$
	\y_t^+ \gets (\y_t+\del)/\eps, \quad
	\y_t^- \gets (\y_t-\del)/\eps
$

\vspace{0.5ex}

$
	\ttil_t \gets
	\frac{1}{2} [\mathtt{GenEst}(\w_t,\x_t',\y_t^+) 
		+ \mathtt{GenEst}(\w_t,\x_t',\y_t^-)]
$

\STATE Return the output $\wbar$ of the algorithm

\end{algorithmic}
\end{algorithm}

%$\rho_\eps(x) = \eps \rho(x/\eps)$
%It is not hard to verify that $\rho_\eps$ is an analytic uniform $\eps$-approximation of the absolute-value function $\abs{x}$, as the following claim states.

In order to approximate the $\del$-insensitive absolute loss function, we define 
\[
	f_\eps(x)
	= \frac{\eps}{2} \rho \left(\frac{x-\del}{\eps} \right) 
		+ \frac{\eps}{2} \rho \left(\frac{x+\del}{\eps} \right) - \del
\]
where $\rho$ is expressed in terms of the error function $\erf$,
\[
	\rho(x) = x \, \erf(x) + \frac{1}{\sqrt{\pi}} e^{-x^2} ,
\]
and consider the approximate loss functions
$
	\losst^\eps(\w) = f_\eps(\w\tr\x_t - \y_t) .
$
Indeed, we have the following.

\begin{claim} \label{cla:svr-approx}
For any $\eps>0$, $f_\eps$ is convex, analytic on the entire real line and
\[
	\sup_{x\in\R} \; \abs{f_\eps(x) - \abs{x}_\del}
	\le \eps \;.
\]
\end{claim}

The claim follows easily from the identity
$
	\abs{x}_\del 
	= \tfrac{1}{2} \abs{x-\del} + \tfrac{1}{2} \abs{x+\del} - \del.
$
In addition, for using Procedure \ref{alg:nlest} we need the following simple observation, 
that follows immediately from the series expansion of $\erf(x)$.

\begin{claim} \label{cla:rho-taylor}

$
	\rho'(x) = \sum_{n=0}^{\infty} a_{2n+1} x^{2n+1} ,
$
with the coefficients $\{a_{2n+1}\}_{n \ge 0}$ given in \eqref{eq:rhocoeffs}.

\end{claim}

We now give the main result of this section, which is a sample complexity bound for the Attribute Efficient SVR (AESVR) algorithm, given in Algorithm \ref{alg:svrreg}.

\begin{theorem} \label{thm:svr}

Assume the distribution $\D$ is such that $\norm{\x}_2 \le 1$ and $\abs{\y} \le B$ with probability $1$.
Then, for any $\w_\st \in \R^d$ with $\norm{\w_\st}_2 \le B$, we have
$
	\E\left[\eloss(\bar{\w})\right]
	\le \eloss(\w_\st) + \eps
$
where $\wbar$ is the output of AESVR (with $\eta$ properly tuned) on a training set of size
\begin{align} \label{eq:svr-m}
	m = \O \left(\frac{d}{k} \right) 
		\cdot \exp \left({\O\left( \log^2{\frac{B}{\eps}} \right)}\right)  .
\end{align}
The algorithm queries at most $k+6$ attributes of each instance in expectation.

\end{theorem}

\begin{proof}

First, note that for the approximate loss functions $\losst^\eps$ we have
\begin{align*}
	\nabla \losst^\eps(\w_t)
	= \tfrac{1}{2} \left[
		\rho'(\w_t\tr\x_t'-\y_t^+) + \rho'(\w_t\tr\x_t'-\y_t^-)
		\right] \cdot \x_t \;.
\end{align*}
Hence, Lemma \ref{lem:nlest} and Claim \ref{cla:rho-taylor} above imply that~$\gtil_t$ in Algorithm \ref{alg:svrreg} is an unbiased estimator of $\nabla \losst^\eps(\w_t)$.
Furthermore, since $\norm{\x_t'}_2 \le \tfrac{1}{\eps}$ and $\abs{\y_t^\pm} \le 2 \tfrac{B}{\eps}$, according to the same lemma we have $\E_t[\ttil_t^2] = \exp(\O(\log^2 {\tfrac{B}{\eps}}))$. Repeating the proof of Lemma~\ref{lem:l2var}, we then have
\begin{align*}
	\E_t[\norm{\gtil_t}_2^2]
	= \E_t[\ttil_t^2] \cdot \E_t[\norm{\xtil_t}_2^2]
%	= e^{\O(\log^2 {\tfrac{B}{\eps}} )} \cdot \tfrac{d}{k} \;.
	= \exp (\O(\log^2 {\tfrac{B}{\eps}})) \cdot \frac{d}{k} \;.
\end{align*}
Replacing $G^2$ in the proof of theorem \ref{thm:l2bound} with the above bound, we get for the output of Algorithm \ref{alg:svrreg},
\begin{align*}
	\EE{\eloss(\wbar)}
	\le \eloss(\w_\st) + \exp (\O(\log^2 {\tfrac{B}{\eps}})) \sqrt{\frac{d}{km}} ,
\end{align*}
which imply that for obtaining an $\eps$-accurate regressor $\wbar$ of the approximate problem, it is enough to take $m$ as given in \eqref{eq:svr-m}. However, claim \ref{cla:svr-approx} now gives that $\wbar$ itself is an $2\eps$-accurate regressor of the original problem, and the proof is complete.
\qedhere

\end{proof}

% subsection support_vector_regression (end)

% section extension_to_general_loss_functions (end)

\section{Experiments} % (fold)
\label{sec:experiments}

In this section we give experimental evidence that support our theoretical bounds, and demonstrate the superior performance of our algorithms compared to the state of the art.
Naturally, we chose to compare our AERR and AELR algorithms
\footnote{The AESVR algorithm is presented mainly for theoretical considerations, and was not implemented in the experiments.}
with the AER algorithm of \cite{cesa2010efficient}. We note that AER is in fact a hybrid algorithm that combines 1-norm and 2-norm regularizations, thus we use it for benchmarking in both the Ridge and Lasso settings.

We essentially repeated the experiments of \cite{cesa2010efficient} and used the popular MNIST digit recognition dataset \cite{lecun1998gradient}. Each instance in this dataset is a $28 \times 28$ image of a handwritten digit $0-9$. We focused on the ``3 vs.~5'' task, on a subset of the dataset that consists of the ``3'' digits (labeled $-1$) and the ``5'' digits (labeled $+1$). We applied the regression algorithms to this task by regressing to the labels.
%\tk{image if possible -- no chance..}

In all our experiments, we randomly split the data to training and test sets, and used 10-fold cross-validation for tuning the parameters of each algorithm. Then, we ran each algorithm on increasingly longer prefixes of the dataset and tracked the obtained squared-error on the test set. For faithfully comparing partial- and full-information algorithms, we also recorded the total \emph{number of attributes} used by each algorithm.

In our first experiment, we executed AELR, AER and (offline) Lasso on the ``3 vs.~5'' task. We allowed both AELR and AER to use only $k=4$ pixels of each training image, while giving Lasso unrestricted access to the entire set of attributes (total of 784) of each instance. The results, averaged over 10 runs on random train/test splits,  are presented in Figure~\ref{fig:lasso}. Note that the $x$-axis represents the cumulative number of attributes used for training. The graph ends at roughly 48500 attributes, which is the total number of attributes allowed for the partial-information algorithms. Lasso, however, completes this budget after seeing merely 62 examples. 

As we see from the results, AELR keeps its test error significantly lower than that of AER along the entire execution, almost bridging the gap with the full-information Lasso. Note that the latter has the clear advantage of being an offline algorithm, while both AELR and AER are online in nature. Indeed, when we compared AELR with an \emph{online} Lasso solver, our algorithm obtained test error almost 10 times better.

In the second experiment, we evaluated AERR, AER and Ridge regression on the same task, but now allowing the partial-information algorithms to use as much as $k=56$ pixels (which amounts to 2 rows) of each instance. The results of this experiment are given in Figure~\ref{fig:ridge}. We see that even if we allow the algorithms to view a considerable number of attributes, the gap between AERR and AER is large.

% some configuration for all plots
\pgfplotsset{
	error bars/error bar style={thick, solid},
	every axis/.append style={ultra thick},
	try min ticks={5},
%	xmajorgrids=true,
%	ymajorgrids=true,
	tick label style={font=\footnotesize},
	label style={font=\footnotesize},
	legend style={font=\footnotesize}
}
\newlength\figurewidth
\newlength\figureheight
\setlength\figurewidth{0.8\columnwidth}
\setlength\figureheight{0.7\figurewidth}

% lasso figure
\begin{figure}[tb]
\begin{center}
%%%% commented out for speed
\input{lasso_3vs5_10pts_k4.tikz}
\caption{Test squared error of Lasso algorithms with $k=4$, over increasing prefixes of the ``3 vs.~5'' dataset.}
\label{fig:lasso}
\end{center}
\end{figure}
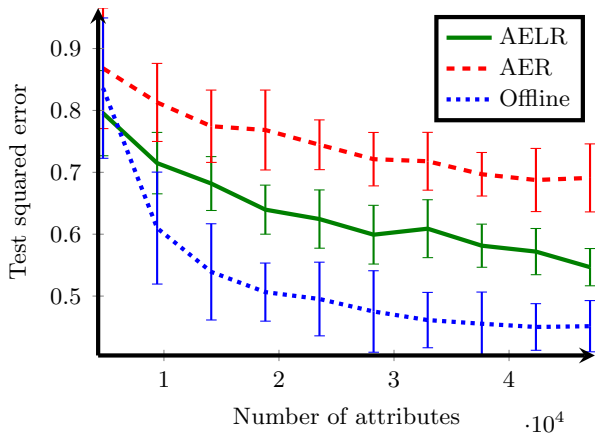

% ridge figure
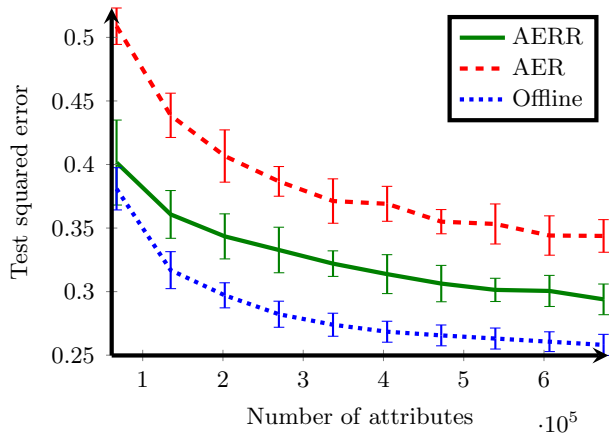
\begin{figure}[tb]
\begin{center}
%%%% commented out for speed
\input{ridge_3vs5_10pts_k56.tikz}
\caption{Test squared error of Ridge algorithms with $k=56$, over increasing prefixes of the ``3 vs.~5'' dataset.}
\vspace{-2ex}
\label{fig:ridge}
\end{center}
\end{figure}

% section experiments (end)

\section{Conclusions and Open Questions} % (fold)
\label{sec:summary_and_conclusions}

We have considered the fundamental problem of statistical regression analysis, and in particular Lasso and Ridge regression, in a setting where the observation upon each training instance is limited to a few attributes, and gave algorithms that improve over the state of the art by a leading order term with respect to the sample complexity. This resolves an open question of \cite{cesa2010efficient}. The algorithms are efficient, and give a clear experimental advantage in previously-considered benchmarks. 

For the challenging case of regression with general convex loss functions, we describe exponential improvement in sample complexity, which apply in particular to support-vector regression. 

It is interesting to resolve the sample complexity gap of $\frac{1}{\eps}$ which still remains for Lasso regression, and to improve upon the pseudo-polynomial factor in $\eps$ for support-vector regression. In addition, establishing analogous bounds for our algorithms that hold with high probability (other than in expectation) appears to be non-trivial, and is left for future work.

Another possible direction for future research is adapting our results to the setting of learning with (randomly) missing data, that was recently investigated---see e.g.~\cite{rostamizadehlearning,loh2011high}. The sample complexity bounds our algorithms obtain in this setting are slightly worse than those presented in the current paper, and it is interesting to check if one can do better.

% section summary_and_conclusions (end)

\section*{Acknowledgments}

We thank Shai Shalev-Shwartz for several useful discussions, and the anonymous referees for their detailed comments.

\bibliographystyle{icml2012}
{
%\small
\bibliography{aeregression}
}

% begin conf. omission
\ifdefined\isconf
\else

%\clearpage
\appendix
%% begin single-column appendix
% \onecolumn

\section{Proofs} % (fold)
\label{sec:proofs}

\subsection{Proof of Lemma \ref{lem:l2var}}
\label{sub:proof_of_lemma_lem:l2var}

It is straightforward to verify that $\E_t[\xtil_{t,r}] = \x_t $ for all $r$, thus also $\E_t[\xtil_t] = \x_t $.
In addition, we have
\begin{align*}
	\E_t[\ttil_t]
	= \sumj \frac{\w_t[j]^2}{\norm{\w_t}_2^2} 
			\cdot \norm{\w_t}_2^2 \frac{\x_t[j]}{\w_t[j]} - \y_t
	= \w_t \tr \x_t - \y_t \;.
\end{align*}
Hence, and since $\xtil_t$ and $\ttil_t$ are independent, we obtain that $\E_t[\gtil_t]$ equals $(\w_t \tr \x_t - y_t) \cdot \x_t$, which is exactly the gradient $\g_t = \nabla \losst(\w_t)$.

Let us turn to bound $\E_t[\norm{\gtil_t}_2^2]$. 
It is easy to verify that $\E_t[\norm{\xtil_{t,r}}_2^2] = d \norm{\x_t}_2^2$ for all $r$, and since $\E_t[\xtil_{t,r}] = \x_t$ and the $\xtil_{t,r}$'s are independent, we have
\begin{align} \label{eq:varx}
	\E_t[\norm{\xtil_t}_2^2]
	&= \frac{1}{k^2} \sum_{r} \E_t[\norm{\xtil_{t,r}}_2^2]
		+ \frac{1}{k^2} \sum_{r \ne s} \norm{\x_t}_2^2 \nonumber \\
	&= \frac{2d+k-1}{2k} \norm{\x_t}_2^2 \;.
\end{align}
This gives the bound $\E_t[\norm{\xtil_t}_2^2] \le 2 d/k$. On the other hand, recalling $\abs{\y_t} \le B$ and using the inequality $(a-b)^2 \le 2(a^2+b^2)$, we obtain
\begin{align} \label{eq:vart}
	\E_t[\ttil_t^2]
%	&\le 2 \norm{\w_t}_2^2 \sumj \x_t[j]^2 + 2 \y_t^2
	\le 2(\norm{\w_t}_2^2 \norm{\x_t}_2^2 + \y_t^2)
	\le 4 B^2 \;.
\end{align}
Finally, from \eqref{eq:varx}, \eqref{eq:vart} and via independence we have $\E_t[\norm{\gtil_t}_2^2] \le 8 B^2 d/k$, and the lemma follows.
\qedhere

% subsection proof_of_lemma_lem:l2var (end)

\subsection{Proof of Lemma~\ref{lem:mwregret}} % (fold)
\label{sub:proof_of_lemma_lem:mwregret}

The lemma is a direct consequence of the following second-order bound for the MW algorithm, which is essentially a simplified version of Lemma II.3 of \cite{chw}.

\begin{lemma} \label{lem:mwregret-gen}

Let $\eta > 0$, and let $\c_1,\ldots,\c_T$ be an arbitrary sequence of vectors in $\R^n$ with $\c_t[i] \ge -1/\eta$ for all $t$ and $i$. Define a sequence $\z_1,\ldots,\z_T$ by letting $\z_1 \gets \mathbf{1}_n$ and for $t \ge 1$,
\begin{align*}
	\z_{t+1}[i] \gets \z_t[i] \cdot \exp(-\eta \c_t[i]), \qquad i=1,\ldots,n .
\end{align*}
Then, for the vectors $\p_t := \z_t / \norm{\z_t}_1$ we have
\begin{align*}
	\sumt \p_t\tr\c_t 
	\le \mini \sumt \c_t[i] 
		+ \frac{\log n}{\eta} + \eta \sumt \p_t\tr\c_t^{2} .
\end{align*}

\end{lemma}

To see how the lemma follows from the above bound, note that we can write the update rule of Algorithm~\ref{alg:l1reg}, in terms of the augmented vectors $\z_t$ and $\gbar_t'$, as follows:
\[
	\z_{t+1}[i] = \z_t[i] \cdot \exp(-\eta \gbar'_t[i])
	\;,\qquad i=1,2,\ldots,2d .
\]
That is, $\z_{t+1}$ is obtained from $\z_t$ by a multiplicative update based on the vector $\gbar'_t$.
Noticing that
$
	\norm{\gbar'_t}_\infty = \norm{\gbar_t}_\infty \le 1/\eta ,
$
we see from Lemma~\ref{lem:mwregret-gen} that for any $i^\st$,
\begin{align*}
	\summ \p_t \tr \gbar'_t
	\le \summ \gbar'_t[i^\st] 
		+ \frac{\log 2d}{\eta} + \eta \summ \p_t \tr (\gbar'_t)^2
\end{align*}
where $\p_t := \z_t / \norm{\z_t}_1$, which gives the Lemma.

For completeness, we provide a proof of Lemma \ref{lem:mwregret-gen}.

\begin{proof}[Proof (of Lemma \ref{lem:mwregret-gen})]

Using the fact that $e^x \le 1+x+x^2$ for $x \le 1$, we have
\begin{align*}
	\norm{\z_{t+1}}_1 
		& = \sum_{i=1}^n \z_t[i] \cdot e^{-\eta \c_t[i]} \\
		& \le \sum_{i=1}^n \z_t[i] \cdot 
			(1 - \eta \c_t[i] + \eta^2 \c_t[i]^2) \\
		 & = \norm{\z_t}_1 \cdot 
			(1 - \eta\p_t \tr \c_t + \eta^2 \p_t \tr \c_t^2)
\end{align*}
and since $e^z \ge 1+z$ for $z \in \R$, this implies by induction that
\begin{align} \label{eq:mwup}
	\log \norm{\z_{T+1}}_1 
	& = \log n + \sumt
	\log(1 - \eta \p_t \tr \c_t + \eta^2 \p_t \tr \c_t^2) \nonumber \\
	& \le \log n - \eta \sumt \p_t \tr \c_t 
		+ \eta^2 \sumt \p_t \tr \c_t^2 \;.
\end{align}
On the other hand, we have
\begin{align} \label{eq:mwlow}
	\log \norm{\z_{T+1}}_1
	&= \log \; \sum_{i=1}^{n} \prodt e^{-\eta \c_t[i]} \nonumber \\
	&\ge \log \; \prodt e^{-\eta\c_t[i^\st]} \nonumber \\
	&= -\eta \sumt \c_t[i^\st] \;. 
\end{align}
Combining~\eqref{eq:mwup} and~\eqref{eq:mwlow} and rearranging, we obtain
\[
	\sumt \p_t\tr\c_t 
	\le \sumt \c_t[i^\st] 
		+ \frac{\log n}{\eta} + \eta \sumt \p_t\tr\c_t^{2}
\]
for any $i^\st$, which completes the proof.
\qedhere

\end{proof}

% subsection proof_of_lemma_lem:mwregret (end)

\subsection{Proof of Lemma \ref{lem:l1lem1}}
\label{proof_of_lemma_lem:l1lem1}

Notice that by our notations,
\[
	\summ \p_t \tr \gbar'_t 
	= \summ \frac{(\z^+_t,\z^-_t)\tr(\gbar_t,-\gbar_t)}
			{\norm{\z^+_t}_1 + \norm{\z^-_t}_1} 
	= \frac{1}{B} \summ \w_t \tr \gbar_t \;,
\]
and
\[
	\min_i \summ \gbar'_t[i]
	= \min_{\norm{\w}_1 \le B} \frac{1}{B} \summ \w \tr \gbar_t 
	\le \frac{1}{B} \summ \w_\st \tr \gbar_t
\]
for any $\w_\st$ with $\norm{\w_\st}_1 \le B$.
Plugging into the bound of Lemma \ref{lem:mwregret}, we get
\[
	\summ \gbar_t \tr (\w_t - \w_\st)
	\le B \left(
		\frac{\log 2d}{\eta} + \eta \summ \p_t \tr (\gbar'_t)^2
		\right) \;.
\]
Finally, taking the expectation with respect to the randomization of the algorithm, and noticing that $\norm{\E_t[(\gbar'_t)^2]}_\infty \le \norm{\E_t[\gtil_t^2]}_\infty \le G^2$, the proof is complete.
\qedhere

\subsection{Proof of lemma~\ref{lem:l1lem2}} % (fold)
\label{sub:l1proofs}

For the proof we need a simple lemma, that allows us to bound the deviation of the expected value of a clipped random variable from that of the original variable, in terms of its variance.

\begin{lemma} \label{lem:clip}
Let $X$ be a random variable with $ \abs{\E[X]} \leq C/2 $ for some $C>0$. Then for the clipped variable $\bar{X} := \clip(X,C) = \max\{\min\{X,C\},-C\}$ we have
\[
	\abs{\E[\bar{X}] - \E[X]} 
		\le 2 \frac{\var[X]}{C} .
\]
\end{lemma}

Now, notice that $\norm{\E_t[\gtil_t^2]}_\infty \le G^2$ implies $\norm{\E_t[\gtil_t]}_\infty \le G$, as
\begin{align*}
	\norm{\E_t[\gtil_t]}_\infty^2
	= \norm{\E_t[\gtil_t]^2}_\infty
	\le \norm{\E_t[\gtil_t^2]}_\infty .
\end{align*}
Since $\gbar[i] = \clip(\gtil[i],\,1/\eta)$ and $\abs{\EE[t]{\gtil_t[i]}} \le G \le 1/2\eta$, the above lemma implies that
\[
	\abs{\EE[t]{\gbar_t[i]} - \EE[t]{\gtil_t[i]}}
	\le 2 \eta \EE[t]{\gtil_t[i]^2}
	\le 2 \eta G^2
\]
for all $i$, which means that $\norm{\EE[t]{\gbar_t - \gtil_t}}_\infty \le 2 \eta G^2$. Together with $\norm{\w_t - \w_\st}_1 \le 2B$, this yields
\[
	\EE[t]{(\gtil_t - \gbar_t) \tr (\w_t - \w_\st)} \le 4 \eta B G^2.
\]
Summing over $t=1,\ldots,m$ and taking the expectation, we obtain the lemma. \qedhere

Finally, we prove the simple lemma.

\begin{proof}[Proof (of lemma \ref{lem:clip})]
As a first step, note that for $x>C$ we have $x-\E[X] \ge C/2$, so that
\[
	C(x-C) \le 2(x-\E[X])(x-C) \le 2(x-\E[X])^2 .
\]
Hence, denoting by $\mu$ the probability measure of $X$, we obtain
\begin{align*}
	\E[X] - \E[\bar{X}]
       & =   \int_{x<-C} (x+C) d\mu + \int_{x>C} (x-C) d\mu
    \\ & \le \int_{x>C} (x-C) d\mu
    \\ & \le \frac{2}{C} \int_{x>C} (x-\E[X])^2 d\mu
    \\ & \le \frac{2}{C} \var[X] .
\end{align*}
Similarly one can prove that $\E[X] - \E[\bar{X}] \ge -2\var[X]/C$, and the result follows.
\end{proof}

\subsection{Proof of Lemma~\ref{lem:l1var}} % (fold)
\label{sub:proof_of_lemma_lem:l1var}

Since $\E_t[\xtil_{t,r}] = \x_t$ and $\E_t[\xtil_{t,r}^2] = d \x_t^2$ for all $r$, we have (by independence) that
\begin{align*}
	\E_t[\xtil_t^2]
	= \frac{1}{k^2} \sum_{r=1}^k \E_t[\xtil_{t,r}^2]
		+ \frac{1}{k^2} \sum_{r \ne s} \x_t^2 
	= \frac{d + k-1}{k} \x_t
\end{align*}
so evidently $\norm{\E_t[\xtil_t^2]}_\infty \le 2 d/k$.
In addition,
\begin{align*}
	\E_t[\ytil_t^2]
	= \norm{\w_t}_1 \sumj \abs{\w_t[j]} \, \x_t[j]^2
	\le \norm{\w_t}_1^2 \norm{\w_t^2}_\infty
	\le B^2 .
\end{align*}
From here we proceed exactly as in Lemma~\ref{lem:l2var}.

% subsection proof_of_lemma_lem:l2var (end)

\subsection{Proof of Lemma \ref{lem:nlest}}

First, denote $\s := \w\tr\x-\y$ and notice that, in any case, we have $\E[\stil_r] = \s$ (see the proof of Lemma \ref{lem:l2var}). Hence, and since $\stil_1,\ldots,\stil_n$ are independent,
\begin{align*}
	\E[\that]
	&= \sum_{n=0}^{\infty} \frac{1}{2^{n+1}}
		\E[ 2^{n+1} a_n \cdot \stil_1 \stil_2 \cdots \stil_n ]
	= \sum_{n=0}^{\infty} a_n \s^n \\
	&= f(\s) \;,
\end{align*}
thus $\that$ is an unbiased estimator of $f'(\s)$. 

For bounding the second moment, let $\nu := 2 \log_2{N}$ and note that if $n \le \nu$ then for all $r$, $\E[\stil_r^2] \le 4 B^2$, and otherwise $\E[\stil_r^2] \le 4B^2 / N \le 1$. 
Also, denoting $f_+'(x) = \sum_{n=0}^\infty \abs{a_n} x^n$ which exists for $\abs{x} \le 1$ by our assumption, we have
$
	\sum_{n=0}^\infty a_n^2 \le (\sum_{n=0}^\infty \abs{a_n})^2 = (f_+'(1))^2 .
$
This yields
\begin{align*}
	\E[\t^2]
	&= 2 \sum_{n=0}^{\infty}
		a_n^2 \E[ 2^{n} \cdot \stil_1^2 \stil_2^2 \cdots \stil_n^2 ] \\
	&\le 2\sum_{n \le \nu} a_n^2 (8 B^2)^n
		+ 2\sum_{n > \nu} a_n^2 \\
	&\le 2 (8 B^2)^{\nu} \sum_{n=0}^\infty a_n^2 
		+ 2 \sum_{n=0}^\infty a_n^2  \\
	&\le (4B)^{2 \nu} \cdot 2(f_+'(1))^2
	= \exp(\O(\log^2 {B}))
\end{align*}

Finally, the expected number of attributes used by the estimator is bounded as follows:
\begin{align*}
	\E[n]
	&= \sum_{n \le \nu} \frac{n}{2^{n+1}}
		+ \sum_{n > \nu} \frac{n N}{2^{n+1}}
	\le \sum_{n=0}^{\infty} \frac{n}{2^{n+1}} 
		+ N \sum_{n \ge \nu} \frac{n}{2^{n+1}} \\
	&\le 1 + N \frac{2 \log_2{N}+1}{N^2} \le 3
\end{align*}
where we have used the identity $\sum_{n=r}^\infty n/2^{n+1} = 2^{-r}(r+1)$.

\section{Lower bounds}

In this section we prove:
\begin{theorem} \label{thm:lb}
Let $\eps = \Omega(1/\sqrt{d})$. Any algorithm for LAO Ridge regression requires to observe at least $\Omega(\frac{d}{ \eps^2})$ coordinates in order to obtain an $O(\eps)$-approximate solution. 
\end{theorem}

\subsection{Information theoretic lower bounds}

Our lower bound is based on the following folklore fact:
\begin{fact}\label{fact:list}
Consider the following random process. Initialize a length-$d$ array $A$ to an
array of $d$ zeros. Choose a random position $i \in [r]$ and set it to be  $A[i] = 1$ with probability $\frac{1}{2}$. With the remaining probability $\tfrac{1}{2}$, set it to be $-1$.  Then any algorithm which determines the value of $A[i]$ 
with probability $\geq \tfrac{2}{3}$ must read $\Omega(d)$ entries of $A$.
\end{fact}

A corollary of this fact is the following more general theorem. 
Consider a $k \times d$
matrix $A$ and the following random process. Pick a subset of coordinates  $T \subseteq  [d]$ of size $|T| = k$. For each  index $j \in [k] $,  set the $j$'th row of $A$ to be $r_i \e_{i}$, where $i = T[j]$ is the $j$'th element of $T$, and $r_i$ is a Rademacher random variable.% that gets $1$ with probability $\frac{1}{2}$ and $-1$ w.p. $\frac{1}{2}$.

\begin{corollary}\label{thm:meta}
Any algorithm that correctly determines the value of $\Omega(k)$ non-zero entries of $A$ with probability $\geq \frac{2}{3}$ must read  $ \Omega(dk)$ entries of $A$.
\end{corollary} 
\ignore{
\begin{proof}
The case $k=1$ is exactly the same as the folklore fact. Assume correctness for some $k$. For $k+2$, the instance of the problem is identical to a $k \times d$-dimensional instance with an addition of a $2 \times (d-k)$ instance (since the new indices can appear anyplace that hasn't been taken yet). Thus  an additional $\Omega(d-k)$ entries are required, for a total of $\sum_{i=1}^k (d-i) = \Omega(dk)$ for $k <<d$. 
\end{proof}
}

\subsection{Proof of Theorem \ref{thm:lb}}

Consider the following Ridge regression setting: The matrix $A$ is created as in the previous subsection with $k = \frac{1}{\eps^2}$. The labels are always one, and the example vectors are chosen uniformly at random as the rows of the matrix $A$. 

The proof will follow from the following two lemmas:

\begin{lemma} \label{lem:shalom1}
There exists a vector $\w^\st$ for which
$$
	\eloss(\w^\st)
	\leq (1-\eps)^2 .
$$
\end{lemma}

\begin{proof}
Consider the vector
$$
	\w^\st
	= \sum_{i \in T} \frac{1}{\sqrt{k} } r_i \e_i
	= \sum_{i \in T} \eps r_i \e_i .
$$
Its expected loss is clearly
$$
	\eloss(\w^*)
	=  (1 - \eps)^2 .
	\qedhere
$$
\end{proof}

\begin{lemma} \label{lem:shalom2}
Let $\del = \frac{1}{100000}$.
Any vector $\w$ for which $ \eloss(\w)  \leq 1- 2(1-\del)\eps$ has $(1-2\del)$ of it's Euclidean weight on the coordinates of $T$, and at least $\frac{1}{10}k$ coordinates in $T$ have weight of at least $\frac{1}{10} \eps$. 
\end{lemma}

\begin{proof}
Let $\w$ be such vector. We have 
\begin{align*}
	1 - 2(1-\del) \eps 
	&\geq \eloss(\w) \\
	&=  \frac{1}{k} \sum_{i \in [T]} (1 - \w^\top \x_i )^2 \\ 
	&=  \frac{1}{k} \sum_{i \in [T]} (1 - 2  \w_i r_i + \w_i^2 ) \\ 
	&\geq  1 - \frac{2}{k} \sum_{i \in T} \w_i r_i .
\end{align*}
Hence
$$
	\sum_{i \in T} \abs{\w_i}
	\geq \sum_{i \in T} \w_i r_i
	\geq (1-\del) {k} \eps
	= \frac{1-\del}{\eps} .
$$
This implies that
$$
	\frac{1}{k} \sum_{i \in T} \w_i^2
	\ge \left( \frac{1}{k} \sum_{i \in T} \abs{\w_i} \right)^2
	\ge (1-\del)^2 \eps^2 ,
$$
hence
$$
	\sum_{i \in T} \w_i^2
	\ge (1-\del)^2
	\ge 1 - 2\del .
$$

Next, we claim that at least $\frac{1}{10}k$ coordinates inside $T$ have weight of at least $\frac{1}{10}\eps$. If that is not the case, then the $\ell_1$ norm of $\w$ supported by $T$ would be upper bounded by two terms:
\begin{enumerate}
\item
The $\ell_1$-norm of the coordinates that are larger than $\frac{1}{10} \eps$. There are few of these, and their $\ell_2$ norm is at most one, hence their $\ell_1$ contribution is at most
$
	\sqrt{\frac{1}{10} k}
	= \frac{1}{\sqrt{10}\eps}.
$
\item
The $\ell_1$-norm of the small coordinates. These are just small, so their total $\ell_1$ contribution is at most $k \cdot \frac{1}{10} \eps = \frac{1}{10 \eps}$. 
\end{enumerate}
Summing both of these up, we arrive at less than $\frac{1-\del}{\eps}$, which is a lower bound for $\ell_1$-norm of the $T$ coordinates. 
\qedhere

\end{proof}

Next, consider any algorithm that attains $ \del \eps = O(\eps)$-approximation for our Ridge regression instance. It finds a vector $\w$ for which 
$$
	\eloss(\w)
	\leq
	\eloss(\w^\st) + \del \eps
%	\leq 1- 2\eps + \del \eps
	\leq  1- 2(1-\del)\eps $$
Where in the first inequality we used Lemma \ref{lem:shalom1}. Thus, by Lemma \ref{lem:shalom2}, the vector $\w$ returned by this algorithm has weight of at least $1 - 2 \del$ on the coordinates of $T$. 

Pick all coordinates of $\w$ with weight larger than $\frac{1}{10}\eps$. According to Lemma \ref{lem:shalom2} there are at least $\frac{1}{10}k$ coordinates inside of $T$ that are this large. In the rest of the coordinates, there are very few coordinates of this magnitude, since the remaining $\ell_2$ weight is $\del$. Hence, there are at most another $\del \cdot \frac{100}{\eps^2} = 100 \del k $ such coordinates elsewhere. 

Overall, we returned a set of coordinates of size $(\frac{1}{10} + 100 \del) k$, from which the vast majority of coordinates are inside $T$. By Corollary \ref{thm:meta}, this requires our algorithm to read $\Omega(kd) = \Omega(\frac{d}{\eps^2})$ samples.

\fi
% end conf. omission

\end{document}

%% file: header.tex
\usepackage{amsthm,amsmath,amssymb}

\newtheorem{theorem}{Theorem}[section]

\newtheorem{lemma}[theorem]{Lemma}
\newtheorem{corollary}[theorem]{Corollary}

\newtheorem{claim}[theorem]{Claim}
\newtheorem{fact}[theorem]{Fact}

\newtheorem*{theorem*}{Theorem}
\newtheorem*{lemma*}{Lemma}
\newtheorem*{corollary*}{Corollary}
\newtheorem*{definition*}{Definition}
\newtheorem*{claim*}{Claim}
\newtheorem*{fact*}{Fact}

\newcommand{\ignore}[1]{}
\newcommand{\E}{\mathbf{E}}
\newcommand{\D}{\mathcal{D}}

\newcommand{\R}{\mathbb{R}}

\newcommand{\EE}[2][]{\E_{#1}\left[ #2 \right]}
\newcommand{\abs}[1]{\lvert #1 \rvert}
\newcommand{\norm}[1]{\lVert #1 \rVert}
\newcommand{\til}[1]{\tilde{#1}}
\newcommand{\tr}{^{{\scriptscriptstyle \top}}}

\renewcommand{\c}{\mathbf{c}}
\newcommand{\e}{\mathbf{e}}

\newcommand{\g}{\mathbf{g}}
\newcommand{\p}{\mathbf{p}}

\renewcommand{\v}{\mathbf{v}}
\newcommand{\w}{\mathbf{w}}

\newcommand{\x}{\mathbf{x}}
\newcommand{\y}{y}

\newcommand{\z}{\mathbf{z}}

\newcommand{\s}{\theta}
\renewcommand{\t}{\phi}

\newcommand{\xtil}{\til{\x}}
\newcommand{\ytil}{\til{\y}}
\newcommand{\gtil}{\til{\g}}
\newcommand{\stil}{\til{\s}}
\newcommand{\ttil}{\til{\t}}

\newcommand{\gbar}{\bar{\g}}
\newcommand{\wbar}{\bar{\w}}

\newcommand{\that}{\hat{\t}}

\newcommand{\yhat}{\hat{\y}}

\newcommand{\erf}{\mathrm{erf}}
\newcommand{\sgn}{\mathrm{sign}}
\newcommand{\var}{\mathbf{var}}
\newcommand{\clip}{\mathrm{clip}}

\renewcommand{\O}{O}

\newcommand{\loss}{\ell}
\newcommand{\losst}{\loss_t}
\newcommand{\eloss}{L_{\D}}

\newcommand{\st}{\star}
\newcommand{\eps}{\varepsilon}
\newcommand{\del}{\delta}

\newcommand{\summ}{\sum_{t=1}^m}

\newcommand{\sumj}{\sum_{j=1}^d}
\newcommand{\sumt}{\sum_{t=1}^{T}}
\newcommand{\prodt}{\prod_{t=1}^{T}}

\newcommand{\mini}{\min_{i\in[d]}}

\renewcommand{\le}{\leqslant}
\renewcommand{\ge}{\geqslant}
\renewcommand{\leq}{\le}
\renewcommand{\geq}{\ge}

%% file: lasso_3vs5_10pts_k4.tikz
% This file was created by matlab2tikz v0.1.4.
% Copyright (c) 2008--2012, Nico Schlömer <nico.schloemer@gmail.com>
% All rights reserved.
% 
% The latest updates can be retrieved from
%   http://www.mathworks.com/matlabcentral/fileexchange/22022-matlab2tikz
% where you can also make suggestions and rate matlab2tikz.
% 
% 
% 
\begin{tikzpicture}

\begin{axis}[%
width=\figurewidth,
height=\figureheight,
scale only axis,
xmin=4280.64, xmax=47463.36,
xlabel={Number of attributes},
ymin=0.404046992871272, ymax=0.965534244271264,
ylabel={Test squared error},
axis lines=left,
axis on top,
legend entries={AELR,AER,Offline},
legend style={nodes=right}]
\addplot [
color=green!50!black,
solid
]
plot [error bars/.cd, y dir = both, y explicit]
coordinates{
 (4704,0.795603720392433)+-(0.0,0.0688195726502734)(9408,0.71485803776341)+-(0.0,0.0496116195373445)(14112,0.681853822801169)+-(0.0,0.0435442504916047)(18816,0.639661727266415)+-(0.0,0.0396309275232841)(23520,0.624456111030838)+-(0.0,0.0470841074400595)(28224,0.599202228753313)+-(0.0,0.047452806706332)(32928,0.608903348720876)+-(0.0,0.0467905010492932)(37632,0.581416811821655)+-(0.0,0.034845708425288)(42336,0.571884915534481)+-(0.0,0.0373306992563666)(47040,0.546725149760739)+-(0.0,0.0300966239896273) 
};

\addplot [
color=red,
dashed
]
plot [error bars/.cd, y dir = both, y explicit]
coordinates{
 (4704,0.868091816143158)+-(0.0,0.0974424281281058)(9408,0.81292588749111)+-(0.0,0.0630775339370052)(14112,0.774449007274695)+-(0.0,0.0584460264704767)(18816,0.768350261782315)+-(0.0,0.0645625633724904)(23520,0.744503598160567)+-(0.0,0.0402051305656428)(28224,0.72124289141191)+-(0.0,0.0431702478426694)(32928,0.717875424084904)+-(0.0,0.0467218743193272)(37632,0.696870050431572)+-(0.0,0.0351735764324453)(42336,0.687579925150294)+-(0.0,0.051029522459949)(47040,0.690987443845537)+-(0.0,0.0549521757923697) 
};

\addplot [
color=blue,
dotted
]
plot [error bars/.cd, y dir = both, y explicit]
coordinates{
 (4704,0.836077573305386)+-(0.0,0.113405998085888)(9408,0.60987865635816)+-(0.0,0.0904555757765982)(14112,0.539113710933935)+-(0.0,0.0777526471163386)(18816,0.506421176252506)+-(0.0,0.0469383061717956)(23520,0.495258624955457)+-(0.0,0.0595237075279841)(28224,0.475098034605834)+-(0.0,0.0658349373279696)(32928,0.461218217180494)+-(0.0,0.044729568563438)(37632,0.455268090065762)+-(0.0,0.05122109719449)(42336,0.450066112375436)+-(0.0,0.0376637002807405)(47040,0.45133097485368)+-(0.0,0.0413831290384844) 
};

\end{axis}
\end{tikzpicture}

%% file: ridge_3vs5_10pts_k56.tikz
% This file was created by matlab2tikz v0.1.4.
% Copyright (c) 2008--2012, Nico Schlömer <nico.schloemer@gmail.com>
% All rights reserved.
% 
% The latest updates can be retrieved from
%   http://www.mathworks.com/matlabcentral/fileexchange/22022-matlab2tikz
% where you can also make suggestions and rate matlab2tikz.
% 
% 
% 
\begin{tikzpicture}

\begin{axis}[%
width=\figurewidth,
height=\figureheight,
scale only axis,
xmin=61355.84, xmax=680308.16,
xlabel={Number of attributes},
ymin=0.249782261492429, ymax=0.523122776734702,
ylabel={Test squared error},
axis lines=left,
axis on top,
legend entries={AERR,AER,Offline},
legend style={nodes=right}]
\addplot [
color=green!50!black,
solid
]
plot [error bars/.cd, y dir = both, y explicit]
coordinates{
 (67424,0.401538737938388)+-(0.0,0.0334280459506332)(134848,0.360732671239329)+-(0.0,0.018787771150739)(202272,0.343458960801848)+-(0.0,0.0177229622556076)(269696,0.332752128949618)+-(0.0,0.0179097208511357)(337120,0.322018411413624)+-(0.0,0.0100565974351034)(404544,0.313765020694926)+-(0.0,0.0153051420119066)(471968,0.306328109372373)+-(0.0,0.0142996737474986)(539392,0.301360520976859)+-(0.0,0.00913335775883484)(606816,0.300531171849021)+-(0.0,0.0122305070989906)(674240,0.293913207584297)+-(0.0,0.0120665456188664) 
};

\addplot [
color=red,
dashed
]
plot [error bars/.cd, y dir = both, y explicit]
coordinates{
 (67424,0.508795342781171)+-(0.0,0.0143274339535313)(134848,0.438679261550782)+-(0.0,0.0174236150929475)(202272,0.406670779206117)+-(0.0,0.0205589143287334)(269696,0.386725570224659)+-(0.0,0.0116811490945351)(337120,0.371250307503383)+-(0.0,0.0174665531466884)(404544,0.369074097308667)+-(0.0,0.0137462105774966)(471968,0.355028801020246)+-(0.0,0.00952911688198654)(539392,0.353217793840303)+-(0.0,0.0157479839576846)(606816,0.344132652086442)+-(0.0,0.0154815200478138)(674240,0.343803181601666)+-(0.0,0.0128359523926403) 
};

\addplot [
color=blue,
dotted
]
plot [error bars/.cd, y dir = both, y explicit]
coordinates{
 (67424,0.38087422584776)+-(0.0,0.0166128489290235)(134848,0.316891738716516)+-(0.0,0.0145407765760314)(202272,0.297090342060935)+-(0.0,0.00986123533097882)(269696,0.282212747206717)+-(0.0,0.0102567217346834)(337120,0.273947695809949)+-(0.0,0.00905984108838453)(404544,0.268453454465837)+-(0.0,0.00823885104211114)(471968,0.265561964887735)+-(0.0,0.00819086544226816)(539392,0.263091659976207)+-(0.0,0.00825895631427925)(606816,0.26061196960304)+-(0.0,0.00777461330917278)(674240,0.258082429417383)+-(0.0,0.00830016792495392) 
};

\end{axis}
\end{tikzpicture}